\documentclass[letterpaper]{article} 
\usepackage[preprint]{aaai2027}  
\usepackage[hyphens]{url}  
\usepackage{graphicx} 
\usepackage{natbib}  
\usepackage{caption} 
\usepackage{booktabs}

\usepackage{amsfonts}
\usepackage{amssymb}
\usepackage{amsmath}

\newtheorem{theorem}{Theorem}

\title{LaPrune: Controllable Differentiable Sparsity at Million Scale}
\author{
    Jakub Antczak\thanks{These authors contributed equally to this work.}\thanks{Wroc{\l}aw University of Science and Technology, Wroc{\l}aw, Poland. Jakub Antczak: \texttt{268745@student.pwr.edu.pl}; Joanna Wojciechowicz: \texttt{255747@student.pwr.edu.pl}.},
    Joanna Wojciechowicz\footnotemark[1]\footnotemark[2],
    {\L}ukasz Struski\thanks{Faculty of Mathematics and Computer Science, Jagiellonian University, Krak\'ow, Poland. {\L}ukasz Struski: \texttt{lukasz.struski@uj.edu.pl}; Jacek Tabor: \texttt{jacek.tabor@uj.edu.pl}.},
    Jacek Tabor\footnotemark[3]
}
\affiliations{
}

\begin{document}

\maketitle

\begin{abstract}
Top-$k$ selection determines which components of a sparse model remain active. Hard selection blocks gradients, while continuous relaxations often couple mask hardness to the selected mass. We introduce LaPrune, a mathematically exact-budget differentiable layer that controls the normalized second moment while preserving the selected mass. A LapSum barrier preserves the selection mass, and a normalized second-moment constraint moves the mask from a dense equal-mass allocation toward hard top-$k$ at each budget. We derive a population prediction of the saturated fraction, a near-binary limiting law, and a tight worst-case guarantee on the near-zero fraction. The normalized hardness parameter is invariant to score scale, while a fixed LapSum temperature is not. 

\end{abstract}


\section{Introduction}

Sparse computation allows neural systems to increase capacity without activating every component. Mixture-of-experts models route each token to a few experts, long-context models retain a limited number of memories, retrieval systems return a short list, and sparse autoencoders activate a small set of latents. Despite their architectural differences, these systems share the same primitive: top-$k$ selection.

At inference, this primitive is a discrete selection. During training, however, a hard top-$k$ gate blocks gradients at the boundary where the model must learn which items to retain. Continuous gates restore gradients, but they must also preserve the operational constraints of the sparse system.

Two independent controls are required. The \emph{budget} specifies a selection mass $k$, corresponding in the hard limit to $k$ active slots such as expert slots, memory entries, retrieved documents, or retained pixels. The \emph{hardness} determines how closely the continuous mask approximates a binary mask. A dense mask can distribute gradients broadly early in training, whereas a near-binary mask can approximate the deployment path later in training. Varying hardness should not change the budget.

Most differentiable top-$k$ relaxations do not provide this separation. Some return soft ranks or soft permutations instead of a directly budgeted mask. Others introduce a temperature or entropy parameter that changes sparsity, but without fixing the selection mass. Some recent linear-time operators obtain their speed by relaxing the budget constraint itself, and in the sparse regime this can produce masks whose mass is far from the intended $k$.

We introduce LaPrune, a differentiable sparse-selection layer that extends the exact LapSum soft top-$k$ operator~\cite{lapsum2025}. LapSum takes scores $r\in\mathbb{R}^n$, a budget $k$, and a user-specified temperature $t$, then returns $p_i=\sigma((r_i-b)/t)$ with a barrier $b$ chosen to satisfy $\sum_i p_i=k$. The temperature controls mask hardness, but its value depends on the scale and distribution of the scores. LaPrune replaces this scale-dependent control with a normalized parameter $\gamma\in[0,1]$. For $0<\gamma<1$, the layer determines a finite $t>0$ and a barrier $b$ so that both $\sum_i p_i=k$ and $\sum_i p_i^2=[a+(1-a)\gamma]k$ hold, where $a=k/n$. We define the endpoints directly: $\gamma=0$ gives the equal-mass mask $p_i=a$, whereas $\gamma=1$ gives the hard top-$k$ indicator. Differentiability through the implicit finite-temperature solution applies to the interior range. The first constraint preserves the selection budget, while the second controls how concentrated the mask is. Thus, unlike LapSum with a manually chosen temperature, LaPrune provides normalized hardness control without changing the budget.

Our contributions are:
\begin{itemize}
\item LaPrune, a mathematically exact-budget differentiable layer whose normalized parameter $\gamma$ controls mask hardness through a second-moment constraint,
\item a comprehensive theoretical analysis of mask saturation, providing both a predictive model for typical score distributions and a strict, distribution-free lower bound on the fraction of near-zero entries,
\item empirical evidence demonstrating that the normalized hardness parameter is robustly scale-invariant and improves feature recovery in sparse networks,
\item a highly efficient implementation that strictly enforces the selection budget while scaling seamlessly up to $n=10^7$.
\end{itemize}
\section{Related Work}
\label{sec:related}

\paragraph{Differentiable Ranking and Top-$k$.}
Prior work addresses several parts of differentiable selection. NeuralSort~\cite{neuralsort2019}, SoftSort~\cite{softsort2020}, differentiable sorting networks~\cite{diffsort2021}, and permutahedron-based operators~\cite{blondel2020} approximate ranks or permutations. Methods designed specifically for top-$k$ selection use entropic optimal transport~\cite{xie2020softtopk}, convex isotonic optimization~\cite{sander2023topk}, perturbation~\cite{berthet2020perturbed}, subset reparameterization~\cite{xie2019reparameterizable}, or sorting networks~\cite{petersen2022differentiable}. These methods primarily construct differentiable approximations of a discrete selection. Our objective is to modulate mask hardness while keeping the selected mass fixed. Sparsemax~\cite{martins2016softmax} and entmax~\cite{peters2019sparse} produce sparse probability vectors on the simplex. They enforce $\sum_i p_i=1$ and have data-dependent support, so neither operator accepts an exact user-specified budget $k$. A matched exact-budget comparison would require an additional calibration that changes their training objective. Scalar-threshold methods are closer to our setting. LapSum~\cite{lapsum2025} enforces the budget through a Laplace-CDF barrier. DFTopK~\cite{dftopk2025} obtains a closed-form threshold by relaxing $\sum_i p_i=k$. The resulting budget drift is measured in the Supplementary Materials. LaPrune adds a second-moment constraint to LapSum and converts temperature into a normalized hardness parameter.

\paragraph{Where Budgeted Selection Is Used.}
The meaning of a sparsity budget depends on the model. Mixture-of-experts systems limit the number of experts assigned to each token~\cite{shazeer2017moe,fedus2022switch}. Network pruning fixes the fraction of retained weights or channels~\cite{han2015,frankle2019lottery}, and token-pruning methods limit the tokens processed by later layers~\cite{rao2021dynamicvit}. Sparse autoencoders constrain the number of active latent features~\cite{makhzani2014ksparse,gao2024topksae}. Despite using different architectures, these applications face the same trade-off. Hard top-$k$ enforces the required budget at the cost of blocking gradients. Soft selection permits gradient-based training, but it may not preserve the requested selection mass for each sample. LaPrune preserves the budget for every selection while controlling how close the mask is to hard top-$k$.

\paragraph{Background: Exact-Budget Soft Top-$k$.}
\label{sec:background}

LapSum~\cite{lapsum2025} takes a score vector $r\in\mathbb{R}^n$, a budget $k\in\{1,\dots,n-1\}$, and a temperature $t>0$. Writing $a=k/n$, it returns a mask $p\in(0,1)^n$ whose entries are determined by a scalar barrier $b$, such that
\begin{equation*}
\begin{aligned}
p_i(t)&=\sigma\!\Big(\frac{r_i-b(t)}{t}\Big),\\
\tfrac1n\sum_{i=1}^n p_i=a
&\quad\Longleftrightarrow\quad
\sum_{i=1}^n p_i=k,
\end{aligned}
\end{equation*}
where $\sigma(u)=\tfrac12 e^{u}$ for $u\le0$ and $1-\tfrac12 e^{-u}$ for $u>0$ is the standard Laplace CDF. For fixed $r$, $k$, and $t$, the map $b\mapsto\sum_i\sigma((r_i-b(t))/t)$ is continuous and strictly decreasing from $n$ to $0$. It therefore has a unique solution satisfying $\sum_i p_i=k$. For each chosen temperature $t$, LapSum computes the corresponding barrier $b$ to preserve the budget. The temperature determines the width of the transition between rejected and selected scores. Our method replaces the direct choice of $t$ with a normalized hardness parameter and determines both $t$ and the corresponding barrier.

\paragraph{Geometric Interpretation of LapSum.}
Equivalently, if $Z_i\sim\mathrm{Laplace}(r_i,t)$, then $p_i=\Pr(Z_i>b)$. The budget equation chooses the shared barrier so that the total tail mass is $k$. Scores near the barrier receive fractional mask values and scores far from it receive values close to zero or one.

Let $g\in\mathbb{R}^n$ denote an upstream gradient. Differentiating the mask with respect to the scores while accounting for the dependence of $b$ on $r$ yields the analytical vector--Jacobian product
\begin{equation}
\label{eq:vjp}
\begin{aligned}
(\nabla_r p)^\top g
&=\kappa\odot\left(g-\langle g,q\rangle\mathbf{1}\right),\\
q_i&=\frac{\kappa_i}{\sum_j \kappa_j},
\qquad
\kappa_i=\frac{1}{2t}e^{-|r_i-b|/t}.
\end{aligned}
\end{equation}
Here $\kappa_i$ is the Laplace density at the barrier for score $r_i$, and $q$ is its normalization across scores. The correction term $\langle g,q\rangle\mathbf{1}$ accounts for the change in $b$ required to preserve $\sum_i p_i=k$. Consequently, the gradient of each mask value depends on all scores through the shared barrier.

LapSum preserves $\sum_i p_i=k$ for every $t>0$. For integer $k$ and pairwise distinct scores, its mask converges to the hard top-$k$ indicator as $t\to0$~\cite[Theorem~3.3]{lapsum2025}.

\section{LaPrune: Exact-Budget Hardness Control}
\label{sec:knob}

LapSum enforces an exact selection budget for every temperature. However, the temperature is not an interpretable measure of mask hardness, because its numerical value depends on the scale and distribution of the input scores. LaPrune replaces this scale-dependent control with a normalized hardness parameter. Given a target hardness, the layer determines the corresponding temperature, while the LapSum barrier continues to enforce the exact budget.

Let the score vector, target budget, and budget fraction be defined by
\begin{equation*}
r\in\mathbb{R}^{n},
\qquad
k\in\{1,\ldots,n-1\},
\qquad
a=\frac{k}{n}.
\end{equation*}
For every positive temperature, the mask entries are
\begin{equation*}
p_i(t)
=
\sigma\left(
\frac{r_i-b(t)}{t}
\right),
\qquad
i=1,\ldots,n,
\end{equation*}
where $\sigma$ is the Laplace cumulative distribution function and the barrier $b(t)$ is the unique solution of
\begin{equation}
\label{eq:laprune-budget}
\sum_{i=1}^{n}p_i(t)=k.
\end{equation}
\paragraph{Second Moment as a Measure of Hardness.}
The budget constraint determines the first moment of the mask
\begin{equation*}
\sum_{i=1}^{n}p_i=k.
\end{equation*}
It does not, however, specify how the mass is distributed among the entries. The same budget can be realized by a uniform mask, in which all entries are equal, or by a binary mask containing exactly $k$ active entries.

We quantify this distinction using the second moment
\begin{equation*}
M_2(p)
:=
\sum_{i=1}^{n}p_i^2.
\end{equation*}
A small second moment indicates that the mask entries are concentrated near their common mean, a diffuse mask. A large second moment indicates that the entries are concentrated near zero and one, a hard, approximately binary mask. 

\paragraph{Feasible Range of the Second Moment.}
The budget and box constraints imply
\begin{equation}
\label{eq:secmom-bounds}
ak=\frac{k^2}{n}
\leq
\sum_{i=1}^{n}p_i^2
\leq k.
\end{equation}
The lower bound is Cauchy--Schwarz and is attained only by the uniform mask $p_i=a$. The upper bound follows from $p_i^2\leq p_i$ and is attained by binary masks with exactly $k$ ones.

\paragraph{Normalized Hardness Parameter.}
The feasible interval permits a scale-free normalization of the second moment. We define
\begin{equation*}
\gamma
:=
\frac{
\sum_{i=1}^{n}p_i^2-ak
}{
k-ak
}
=
\frac{
\sum_{i=1}^{n}p_i^2-ak
}{
k(1-a)
},
\qquad
\gamma\in[0,1].
\end{equation*}
Equivalently, setting the normalized hardness imposes the second-moment constraint
\begin{equation*}
\sum_{i=1}^{n}p_i^2
=
\beta k,
\qquad
\beta
=
a+(1-a)\gamma,
\qquad
\beta\in[a,1].
\end{equation*}
Jointly rescaling the scores, barrier, and temperature leaves the mask, and therefore $\gamma$, unchanged. Because LaPrune solves for the temperature, a rescaling of the input scores produces the corresponding rescaled solution. A fixed numerical temperature does not have this invariance.

\paragraph{Temperature Determined by the Hardness Constraint.}
For every positive temperature, let the barrier be defined by Equation~\eqref{eq:laprune-budget}, and define
\begin{equation*}
M_2(t)
:=
\sum_{i=1}^{n}
\left(
\sigma\left(
\frac{r_i-b(t)}{t}
\right)
\right)^2.
\end{equation*}

\begin{theorem}[Unique Temperature]
\label{theorem:unique-temperature}
For every nonconstant score vector, $M_2(t)$ is continuous and strictly decreasing in $t$. If the scores are pairwise distinct, then $M_2(t)\to k$ as $t\downarrow0$ and $M_2(t)\to ak$ as $t\to\infty$. Consequently, every $\gamma\in(0,1)$ determines a unique temperature $t_\gamma$ satisfying the target second moment.
\end{theorem}

The proof is given in the Supplementary Materials section "Monotonicity and Uniqueness of the Temperature." Increasing $\gamma$ decreases $t_\gamma$. The finite-temperature family is defined for $\gamma\in(0,1)$. Its limits as $\gamma\to0$ and $\gamma\to1$ are the uniform and hard top-$k$ masks, respectively. We therefore define the endpoints explicitly as $p_i=a$ at $\gamma=0$ and the hard top-$k$ indicator at $\gamma=1$.

\paragraph{Computing LaPrune: Joint Solves and Implicit Gradients.}
\label{sec:computing-laprune}

For $0<\gamma<1$, LaPrune must determine both the barrier $b$,
which preserves the selection mass, and the temperature $t$, which
realizes the requested hardness. We set $\tau=\log t$ and determine
the barrier and temperature simultaneously by solving
\begin{align}
F_1(b,\tau)
&:= \sum_{i=1}^{n} p_i-k = 0,
\nonumber\\
F_2(b,\tau)
&:= \sum_{i=1}^{n} p_i^2-\beta k = 0,
\label{eq:laprune-joint-system}
\end{align}
where
\begin{equation*}
p_i
=
\sigma\!\left(\frac{r_i-b}{t}\right),
\qquad
\beta=a+(1-a)\gamma,
\qquad
a=\frac{k}{n}.
\end{equation*}
To solve the system in \eqref{eq:laprune-joint-system}, let
\begin{equation*}
z_i=\frac{r_i-b}{t},
\qquad
d_i=\sigma'(z_i),
\end{equation*}
and define the four reductions
\begin{equation*}
\begin{alignedat}{2}
S_0 &:= \sum_i d_i,
&\qquad
S_1 &:= \sum_i d_i z_i,
\\
U_0 &:= \sum_i p_i d_i,
&\qquad
U_1 &:= \sum_i p_i d_i z_i.
\end{alignedat}
\end{equation*}
The Jacobian of the constraints with respect to $(b,\tau)$ is
\begin{equation}
J_{(b,\tau)}F
=
\begin{bmatrix}
-S_0/t & -S_1 \\
-2U_0/t & -2U_1
\end{bmatrix}.
\label{eq:laprune-constraint-jacobian}
\end{equation}
Using the Jacobian in \eqref{eq:laprune-constraint-jacobian}, each
Newton iteration computes the increment $(\Delta b,\Delta\tau)$ from
\begin{equation}
J_{(b,\tau)}F
\begin{bmatrix}
\Delta b \\
\Delta \tau
\end{bmatrix}
=
-
\begin{bmatrix}
F_1 \\
F_2
\end{bmatrix}.
\label{eq:laprune-newton-update}
\end{equation}
The parameters are then updated as
$b\leftarrow b+\Delta b$ and
$\tau\leftarrow\tau+\Delta\tau$.
Each iteration requires only elementwise operations, four reductions
over the scores, and the solution of a $2\times2$ linear system.

We initialize the barrier using an estimate of the $(1-a)$-quantile
of the score distribution. Under a Gaussian approximation, the
initial barrier is
\begin{equation*}
b_0
=
\widehat{\mu}
+
\widehat{s}\,\Phi^{-1}(1-a),
\end{equation*}
where $\widehat{\mu}$ and $\widehat{s}$ are the row-wise mean and
standard deviation. Using the near-binary approximation, we initialize
the temperature as
\begin{equation*}
t_0
\approx
\frac{4a(1-a)(1-\gamma)}
     {3\widehat{\rho}(b_0)},
\end{equation*}
where $\widehat{\rho}(b_0)$ denotes the corresponding density estimate
at the initial barrier. Large Newton updates are clipped, and a nested
bisection solver provides a robust fallback when the residuals do not
reach the prescribed tolerance. The endpoints $\gamma=0$ and
$\gamma=1$ bypass the finite-temperature solver and return the uniform
and hard top-$k$ masks, respectively.

For the backward pass, gradients are obtained by differentiating the
converged solution directly, without unrolling the Newton iterations.
For any scalar input $x\in\{r_1,\ldots,r_n,a,\gamma\}$, the implicit
function theorem gives
\begin{equation}
J_{(b,\tau)}F
\begin{bmatrix}
\partial b/\partial x \\
\partial \tau/\partial x
\end{bmatrix}
=
-
\left.
\frac{\partial F}{\partial x}
\right|_{b,\tau}.
\label{eq:laprune-implicit-differentiation}
\end{equation}
The system in \eqref{eq:laprune-implicit-differentiation} reuses the
same $2\times2$ constraint Jacobian. The resulting vector--Jacobian
product is evaluated analytically, so the backward pass does not need
to store the solver trajectory and retains only the score vector and
the quantities computed at the converged solution.

\section{Theoretical Analysis of the Hardness Parameter}
\label{sec:theory}

The hardness parameter $\gamma$ moves the mask from a dense allocation toward a near-binary top-$k$ mask. We characterize this transition through the fraction of mask entries that saturate to $0$ or $1$, and we do so in two complementary ways: a population prediction under a score model and a lower bound that holds for any score distribution. Throughout, let the scores follow a density $\rho$ with CDF $F$, and let $\mu(x)=\sigma((x-b)/t)$ denote the population mask, the value the layer assigns to a score equal to $x$. With $a=k/n$ and $\beta=a+(1-a)\gamma$, the budget and hardness constraints fix the first two population moments of the mask at $a$ and $\beta a$.

\paragraph{Mean-Field Prediction.}
The mean-field prediction replaces the $n$ random scores by their population density $\rho$ and works with the expected moments, which the empirical moments approach as $n$ grows. The pair $(t,b)$ then solves the two-moment fixed point
\begin{equation}
\label{eq:fixedpoint}
\begin{aligned}
\bar M_1(t,b)&:=\int\sigma\!\Big(\tfrac{x-b}{t}\Big)\rho(x)\,dx=a,\\
\bar M_2(t,b)&:=\int\left(\sigma\!\Big(\tfrac{x-b}{t}\Big)\right)^2 \rho(x)\,dx=\beta a.
\end{aligned}
\end{equation}
An entry is treated as numerically zero when $\mu(x)<\varepsilon$, where we use $\varepsilon=10^{-3}$. For $0<\varepsilon<1/2$ and $x\le b$, the Laplace CDF satisfies $\sigma(u)=\tfrac12 e^u$. Therefore,
\[
\mu(x)<\varepsilon
\iff
\tfrac12 e^{(x-b)/t}<\varepsilon
\iff
x<b+t\ln(2\varepsilon).
\]
Similarly,
\[
\mu(x)>1-\varepsilon
\iff
x>b-t\ln(2\varepsilon).
\]
Thus, under the population model, the fractions of entries close to zero and one are
\begin{equation}
\label{eq:sat}
\begin{aligned}
\mathrm{frac}_0
&=F\!\left(b+t\ln(2\varepsilon)\right),\\
\mathrm{frac}_1
&=1-F\!\left(b-t\ln(2\varepsilon)\right),\\
s_\varepsilon
&=\mathrm{frac}_0+\mathrm{frac}_1.
\end{aligned}
\end{equation}
The two conditions above place the nonsaturated values, those with $\varepsilon\le\mu(x)\le1-\varepsilon$, in the band $\left[\,b-t\lvert\ln(2\varepsilon)\rvert,\ b+t\lvert\ln(2\varepsilon)\rvert\,\right]$, of width $2t\lvert\ln(2\varepsilon)\rvert$. Increasing $\gamma$ decreases $t$, which narrows this band and moves the mask values closer to $0$ and $1$. We solve the two-moment system~\eqref{eq:fixedpoint} numerically for $(t,b)$. The Gaussian experiments use numerical quadrature. Closed-form Laplace moments and implementation details are given in the Supplementary Materials section "Numerical and Analytical Details."

\paragraph{Near-Binary Limit Under Regularity.}
Assume that $F$ has a continuous and positive density in a neighborhood of the unique $(1-a)$-quantile
\[
q_\star:=q_{1-a}.
\]
As $\gamma\to1$, we have $t\to0$. The budget constraint $\bar M_1=a$ then reduces to $1-F(b)=a$, so $b\to q_\star$. Moreover,
\[
\bar M_1-\bar M_2
=
\int_{\mathbb{R}}
\sigma\!\left(\frac{x-b}{t}\right)
\left[1-\sigma\!\left(\frac{x-b}{t}\right)\right]
\rho(x)\,dx.
\]
After the change of variables $u=(x-b)/t$, the first-order expansion is
\[
\bar M_1-\bar M_2
=t\rho(q_\star)I+o(t),
\qquad
I:=\int_{\mathbb{R}}\sigma(u)(1-\sigma(u))\,du.
\]
For the unit Laplace CDF, $I=\tfrac34$. Since
\[
\bar M_1-\bar M_2
=a-\beta a
=a(1-a)(1-\gamma),
\]
we obtain
\begin{equation}
\label{eq:tasympt}
t
\approx
\frac{a(1-\beta)}{I\rho(q_\star)}
=
\frac{4}{3}
\frac{a(1-a)(1-\gamma)}{\rho(q_{1-a})}.
\end{equation}

The nonsaturated fraction is the probability mass inside the transition interval:
\[
1-s_\varepsilon
=
F\!\left(b-t\ln(2\varepsilon)\right)
-
F\!\left(b+t\ln(2\varepsilon)\right).
\]
For small $t$,
\[
1-s_\varepsilon
\approx
2t\lvert\ln(2\varepsilon)\rvert\rho(q_\star).
\]
Substituting Equation~\eqref{eq:tasympt} cancels the density term and gives
\begin{equation}
\label{eq:fbinasympt}
1-s_\varepsilon
\approx
\frac{8}{3}
\lvert\ln(2\varepsilon)\rvert
a(1-a)(1-\gamma).
\end{equation}
Therefore, near the binary limit, the nonsaturated fraction decreases linearly with $1-\gamma$. Its leading-order coefficient depends on $a$ and $\varepsilon$, but not on the local score density. This result holds for continuous score distributions with a positive density at $q_{1-a}$. We make no density-independent claim for discrete score distributions, atoms at the quantile, or a vanishing density at $q_{1-a}$.

\paragraph{Worst-Case Guarantee for Any Distribution.}
The mean-field analysis describes typical, large-sample behavior. We next establish a worst-case counterpart: a deterministic lower bound on the fraction of near-zero mask entries that holds for \emph{every} score distribution and depends only on the first two mask moments. After sorting the mask values, let $f:[0,1]\to[0,1]$ denote the resulting nondecreasing rank profile. It satisfies
\[
\int_0^1 f(x)\,dx=a,
\qquad
\int_0^1 f(x)^2\,dx
=a\bigl(a+(1-a)\gamma\bigr).
\]

\begin{theorem}[Worst-Case Sparsity Floor]
\label{thm:worst}
Let $f:[0,1]\to[0,1]$ be nondecreasing and satisfy
\[
\begin{aligned}
\int_0^1 f(x)\,dx&=a,
&\int_0^1 f(x)^2\,dx&=a\beta,\\
\beta&=a+(1-a)\gamma.
\end{aligned}
\]
Then, for every $\varepsilon\in(0,\beta]$,
\[
\bigl|\{x\in[0,1]:f(x)<\varepsilon\}\bigr|
\ge
H_\star(\varepsilon),
\]
where
\[
H_\star(\varepsilon)
:=(1-a)
\max\left\{0,1-\frac{a(1-\gamma)}{\varepsilon}\right\}.
\]
This bound is tight over the stated moment class.
\end{theorem}

\noindent The proof and extremizing constructions are given in the Supplementary Materials section "Proof of the Worst-Case Sparsity Floor."

The lower bound is zero when $\varepsilon\le a(1-\gamma)$, because the moment constraints can then be satisfied while keeping every mask value at least $\varepsilon$. The bound increases with $\gamma$. As $\gamma\to1$, it converges to $1-a$ for every fixed $\varepsilon>0$, recovering the fraction of zeros in a hard top-$k$ mask.


\section{Empirical Evaluation}
\label{sec:experiments}

We test five claims. First, LaPrune should retain LapSum's favorable empirical scalability. Second, $\gamma$ should control a normalized second moment while the budget remains fixed and provide a monotone soft-to-hard path at each budget. Third, the soft-to-hard path should improve recovery when the selected features are known. Fourth, budget-preserving training should remain useful in a sparse model with real activations. Fifth, we numerically verify that the generated masks respect the distribution-free worst-case sparsity floor established in Theorem~\ref{thm:worst}. The budget equation is exact mathematically and is enforced to numerical tolerance in the implementation. The normalized budget and second-moment residuals are approximately $10^{-7}$ in our diagnostic sweep. Additional results appear in the Supplementary Materials sections "Additional Hardness Diagnostics," "Feature-Selection Budget Sweep," "Sparse-Autoencoder Sweeps," "CIFAR-100 Top-$k$ Classification," and "Numerical and Analytical Details."

\subsection{Computational Scaling and Memory}
\label{sec:scaling}

\begin{figure}[h]
  \centering
  \includegraphics[width=\linewidth]{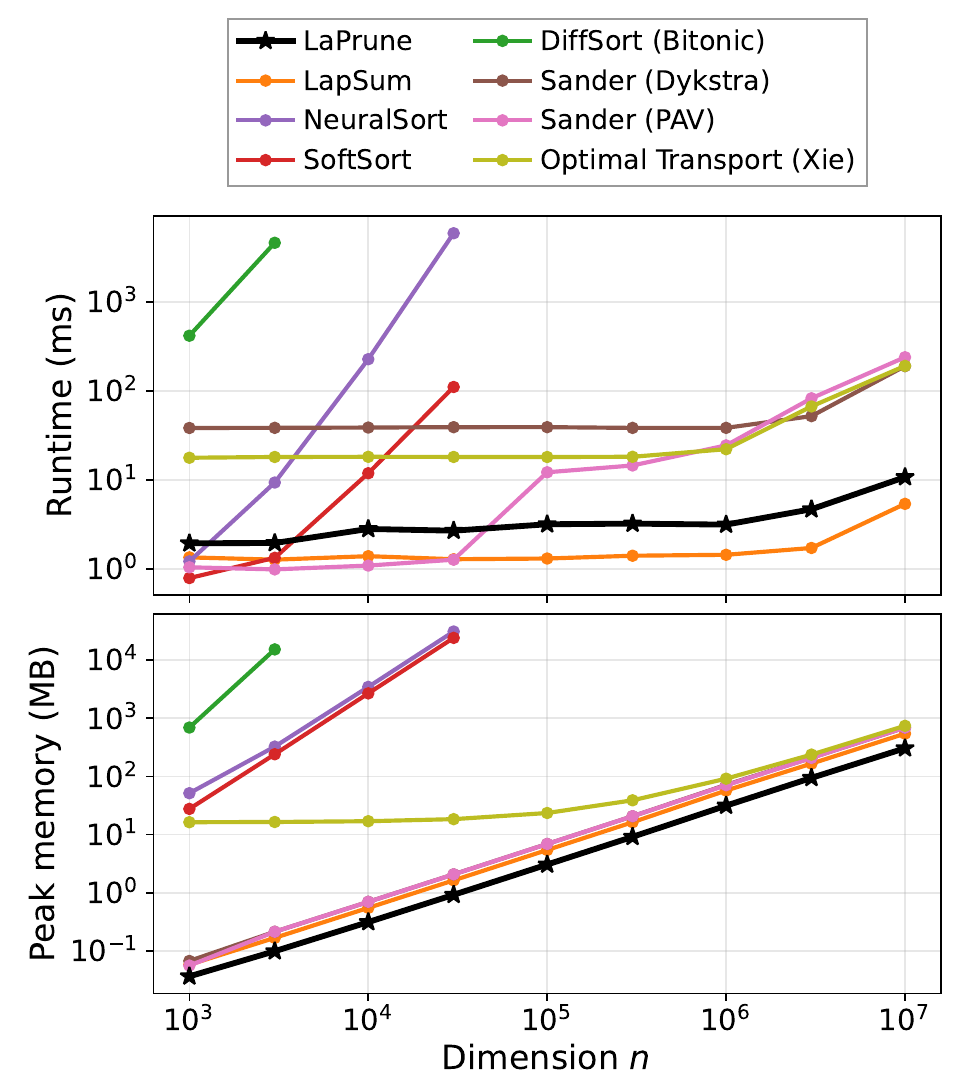}
  \caption{Combined forward-and-backward runtime (top) and peak memory (bottom) against score dimension $n$ on an A100, with $k=n/16$ and $\gamma=0.9$. LaPrune and LapSum reach $n=10^7$ while preserving per-sample mass to numerical tolerance.}
  \label{fig:scaling}
\end{figure}

We measure combined forward and backward runtime and peak GPU memory for dimensions up to $n=10^7$ at $k=n/16$ and $\gamma=0.9$ on an NVIDIA A100. LaPrune and LapSum reach $n=10^7$ while enforcing the per-sample mass constraint to numerical tolerance. Sorting relaxations exhaust memory at substantially smaller dimensions.

At $n=10^7$, LaPrune takes $10.75$ ms and 305 MB, compared with $5.38$ ms and 544 MB for LapSum. The second constraint therefore adds a factor of two in runtime while reducing measured peak memory. Sander-PAV and optimal transport reach the same dimension but are $18$--$22\times$ slower.

Each point uses a single float32 score vector, one warm-up call, and the median of 3--15 synchronized repetitions in an isolated subprocess. LaPrune uses its analytical implicit backward pass, whereas the baseline backward implementations follow their respective reference or local implementations. The full timing protocol and implementation details are given in the Supplementary Materials section "Scaling Protocol and Implementations."

\begin{figure*}[!t]
  \centering
  \includegraphics[width=0.9\textwidth]{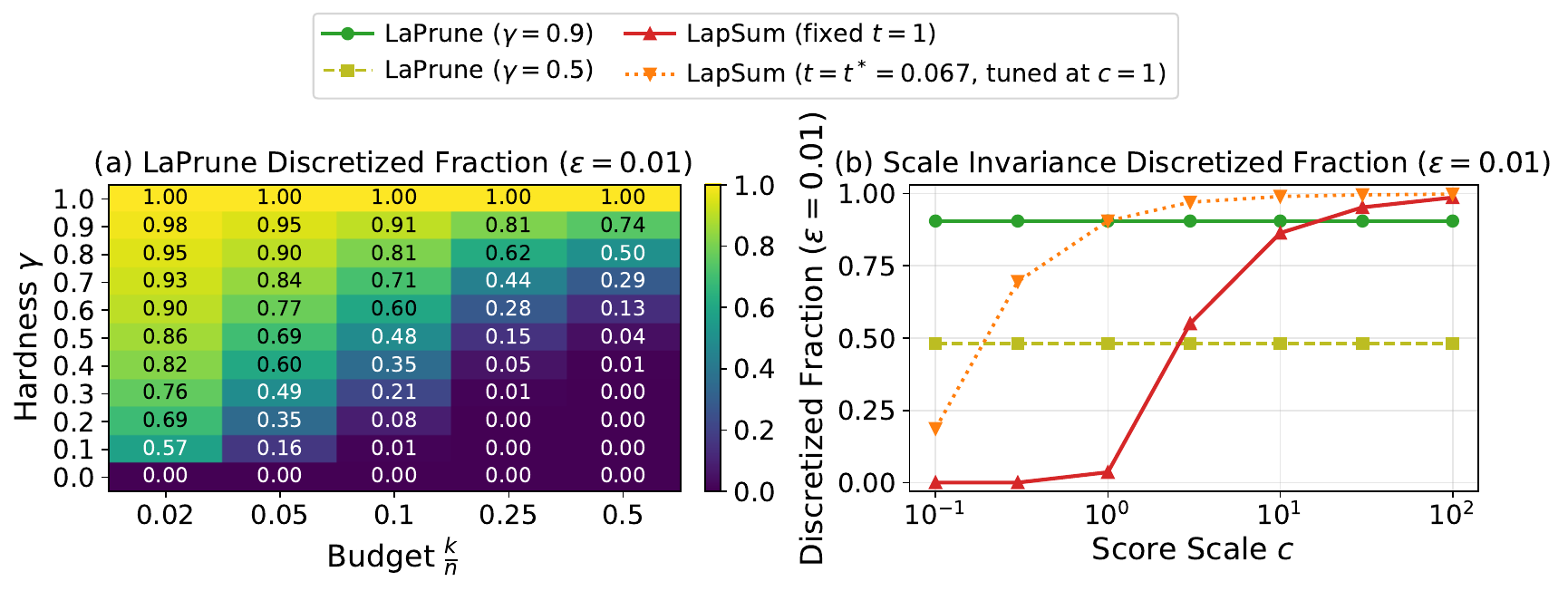}
  \caption{Normalized second-moment control on Gaussian scores with $n=2000$, averaged over five seeds. The vertical quantity is the fraction of entries within $\varepsilon=0.01$ of zero or one. Left: LaPrune provides a monotone path between the equal-mass and binary endpoints at each fixed budget $k/n$. Equal $\gamma$ need not yield equal saturation across budgets. Right: LaPrune keeps the same discretized fraction when scores are rescaled by $c$, while fixed or once-tuned LapSum temperatures drift.}
  \label{fig:discretization}
\end{figure*}

\subsection{Normalized Hardness and Budget Preservation}
\label{sec:hardness-exp}

Figure~\ref{fig:discretization} tests whether $\gamma$ provides a normalized second-moment control that cannot be reproduced by a fixed LapSum temperature. We draw $n=2000$ Gaussian scores, average over five seeds, and define an entry as discretized when $p_i<0.01$ or $p_i>0.99$. The points at $\gamma=0$ and $\gamma=1$ are evaluated using the explicit equal-mass and hard top-$k$ definitions, not a finite-temperature solve. At every tested budget, these endpoints give an equal-mass mask and a binary mask, respectively. At $\gamma=0.9$, the discretized fraction ranges from $0.74$ at $k/n=0.5$ to $0.98$ at $k/n=0.02$. Thus, equal values of $\gamma$ do not imply equal saturation across budgets. Instead, $\gamma$ parameterizes a monotone path between the two endpoints for each fixed budget.

The right panel rescales the same scores by $c\in[0.1,100]$. LaPrune holds the discretized fraction at $0.481$ for $\gamma=0.5$ and $0.905$ for $\gamma=0.9$. LapSum with $t=1$ changes from $0.000$ to $0.987$. A temperature tuned to match LaPrune at $c=1$ also drifts away from the target. This experiment isolates the role of the normalized second moment: the solved $t$ follows score scale, while $\gamma$ retains its meaning.

\begin{table*}[t]
  \centering
  \caption{Feature selection with 200 features, 10 informative features, and exact hard top-10 deployment. Values are mean $\pm$ standard deviation over 20 seeds. The near-binary fraction is the fraction of training-time mask entries within $\varepsilon=10^{-3}$ of zero or one.}
  \label{tab:feature-selection}
  \small
  \begin{tabular}{lccc}
    \toprule
    Method & Recovery F1 & Deploy ACC & Near-Binary Fraction \\
    \midrule
    LaPrune ($\gamma{:}\,0\!\to\!0.9$) (ours)& \textbf{0.855 $\pm$ 0.097} & 0.872 $\pm$ 0.028 & 0.869 $\pm$ 0.029 \\
    LaPrune ($\gamma=0.9$) (ours)& 0.800 $\pm$ 0.118 & \textbf{0.882 $\pm$ 0.028} & 0.908 $\pm$ 0.030 \\
    LapSum ($t=1$) & 0.795 $\pm$ 0.116 & 0.807 $\pm$ 0.037 & 0.000 $\pm$ 0.000 \\
    hard top-$k$ (STE) & 0.375 $\pm$ 0.155 & 0.743 $\pm$ 0.097 & 1.000 $\pm$ 0.000 \\
    \midrule
    DFTopK & 0.535 $\pm$ 0.106 & 0.852 $\pm$ 0.048 & 0.000 $\pm$ 0.000 \\
    \bottomrule
  \end{tabular}
\end{table*}


Two checks separate hardness control from numerical error. On $n=10^4$ scores, the normalized budget and second-moment residuals remain below approximately $10^{-7}$, so LaPrune and LapSum preserve the requested budget to numerical tolerance. On $n=2\times10^5$ scores with $\varepsilon=10^{-3}$, the population prediction in Equations~\eqref{eq:fixedpoint}--\eqref{eq:sat} matches the empirical saturated fraction with maximum error $0.0022$. By contrast, DFTopK reaches $39$ times the target mass for $k=n/100$ at $\tau=5$. The corresponding figures are in the Supplementary Materials section Additional Hardness Diagnostics."


\subsection{Differentiable Feature Selection}
\label{sec:feature-selection}

We construct 200-dimensional data with exactly ten informative features, using 1500 training and 500 test samples. A learned score for each feature passes through the tested mask into the same linear classifier. Every trained model is evaluated with an exact hard top-10 mask. Feature-recovery F1 directly measures whether the selected set matches the known informative set.

Table~\ref{tab:feature-selection} reports mean and standard deviation across 20 paired seeds. The annealed schedule increases $\gamma$ linearly by epoch from $0$ at epoch 0 to $0.9$ at epoch 149 of 150. This schedule gives F1 $0.855\pm0.097$, compared with $0.795\pm0.116$ for LapSum. The paired difference is $0.060\pm0.066$ with Wilcoxon $p=0.005$. The gain over fixed $\gamma=0.9$ is $0.055\pm0.067$ with $p=0.007$. DFTopK obtains recovery F1 $0.535\pm0.106$ after hard top-10 deployment.


\begin{table}[h]
  \centering
  \caption{Sparse autoencoders on ResNet-18/CIFAR-100 activations at
  dictionary size $m=4096$ and deployment budget $k=32$.
  Values are mean $\pm$ standard deviation over eight seeds.}
  \label{tab:sae}
  \small
  \begin{tabular}{@{}p{0.41\columnwidth}cc@{}}
    \toprule
    Method & FVU & Probe ACC (\%) \\
    \midrule
    LaPrune ($\gamma{:}\,0\!\to\!0.9$) (ours)
      & 0.3587 $\pm$ 0.0006
      & 56.59 $\pm$ 0.36 \\
    LaPrune ($\gamma=0.9$) (ours)
      & \textbf{0.3462 $\pm$ 0.0002}
      & \textbf{57.19 $\pm$ 0.40} \\
    LapSum ($t=1$)
      & 7.906 $\pm$ 0.123
      & 45.35 $\pm$ 0.33 \\
    hard top-$k$ (STE)
      & 1.155 $\pm$ 0.010
      & 46.89 $\pm$ 0.39 \\
    \midrule
    DFTopK
      & 0.5825 $\pm$ 0.0005
      & 51.03 $\pm$ 0.15 \\
    \bottomrule
  \end{tabular}
\end{table}

\subsection{Sparse Autoencoders on Real Activations}
\label{sec:sae}

\begin{figure*}[t]
  \centering
  \includegraphics[width=0.9\textwidth]{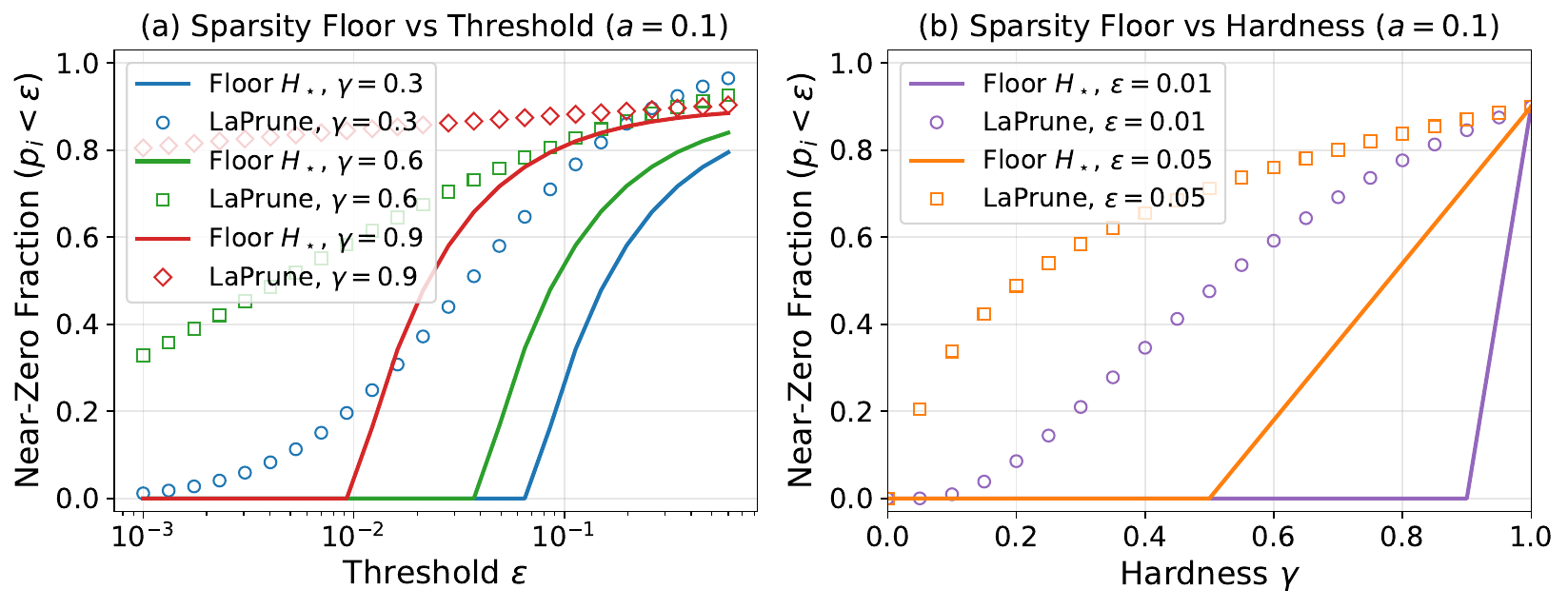}
  \caption{Direct numerical check of the worst-case sparsity floor at budget fraction $a=1/10$ on $n=2\times10^5$ Gaussian scores, averaged over five seeds. Open markers show the measured near-zero fraction of LaPrune masks and solid lines show the distribution-free lower bound $H_\star(\varepsilon)$ from Theorem~\ref{thm:worst}. Left: the threshold $\varepsilon$ varies at fixed hardness. Right: the hardness $\gamma$ varies at fixed threshold. Every measured point lies at or above its corresponding floor.}
  \label{fig:supp-worst-floor}
\end{figure*}

We train overcomplete sparse autoencoders on frozen ResNet-18 activations from CIFAR-100. The activation dimension is 512, the dictionary has $m=4096$ latents, and the deployment budget is $k=32$. Each model is trained for 30 epochs on 50,000 activations using mean-squared reconstruction error without an auxiliary dead-latent loss. The annealed schedule increases $\gamma$ linearly from $0$ at epoch 0 to $0.9$ at epoch 23, then holds it fixed through epoch 29. We report eight-seed results on 10,000 held-out activations. Reconstruction uses fraction of variance unexplained (FVU). A linear probe trained on the exact hard top-32 codes measures retained class information.


Table~\ref{tab:sae} shows that both LaPrune schedules outperform the other measured budget-preserving operators. Fixed $\gamma=0.9$ obtains FVU $0.3462\pm0.0002$ and probe accuracy $57.19\pm0.40\%$. Hard top-$k$ reaches FVU $1.155\pm0.010$ and probe accuracy $46.89\pm0.39\%$. LapSum remains diffuse at $t=1$, leading to FVU $7.906\pm0.123$. DFTopK obtains lower FVU than these two baselines. Simplex operators are not included because their training masks have mass one and data-dependent support. Forcing those models through hard top-32 deployment would confound operator quality with a change in mask scale and support size. In paired comparisons, fixed LaPrune improves both FVU and probe accuracy in all eight runs against LapSum, hard top-$k$, and DFTopK ($p=0.0078$ for each Wilcoxon test).


\subsection{Worst-case sparsity floor}
We directly check Theorem~\ref{thm:worst} on masks produced by the implemented operator. We draw $n=2\times10^5$ Gaussian scores, fix the budget fraction at $a=1/10$, and average the near-zero fraction over five seeds. Figure~\ref{fig:supp-worst-floor} compares the measured fraction $\Pr(p_i<\varepsilon)$ with the distribution-free floor
\[
H_\star(\varepsilon)
=(1-a)\max\left\{0,1-\frac{a(1-\gamma)}{\varepsilon}\right\}.
\]
The measured fraction remains at or above the floor for every tested threshold and hardness value. The Gaussian masks need not attain equality because the theorem is tight over the full moment class and not for every score distribution.


\section{Conclusion}

This work develops a mathematical framework for controlling the softness of differentiable top-$k$ masks without changing their prescribed mass. LaPrune expresses budget and hardness as separate constraints, links the normalized second moment to mask saturation, and characterizes the transition from equal allocation to discrete selection. The resulting theory includes a population description, an asymptotic law near the binary limit, and a distribution-free sparsity guarantee.

The experiments are designed as controlled tests of these mathematical claims, not as an effort to maximize performance on individual benchmarks. They show that the proposed parameter retains its interpretation under score rescaling, that the constrained masks can support learning in synthetic and real-feature settings, and that the operator remains computationally practical at large dimensions. Their role is therefore to establish that the theoretical construction is numerically realizable and relevant to sparse learning.

\paragraph{Limitations.}
The present formulation uses a Laplace-CDF mask and controls hardness through one normalized moment. Its guarantees describe mask geometry, not the complete optimization dynamics or downstream generalization of models trained with the layer. Numerical conditioning also becomes more difficult close to the binary endpoint. Extending the analysis to other mask families, alternative concentration constraints, and training-dependent score distributions may yield a broader theory of differentiable sparsity. Improving endpoint solvers and studying the framework in larger sparse architectures are natural directions for future work.

\bibliography{aaai2027}

\clearpage
\appendix

\section{Supplementary Materials}

\subsection{Additional Hardness Diagnostics}

\paragraph{Mean-field agreement.}
We draw $n=2\times10^5$ Gaussian scores and sweep $\gamma$ at budget fractions $a\in\{1/2,1/10\}$. An entry is saturated when $p_i<\varepsilon$ or $p_i>1-\varepsilon$, with $\varepsilon=10^{-3}$. Figure~\ref{fig:supp-mean-field} overlays finite-sample measurements with the population fixed-point and saturation equations from the main paper. The maximum absolute discrepancy is $0.0022$ for $a=1/2$ and $0.0006$ for $a=1/10$.

\begin{figure}[h]
  \centering
  \includegraphics[width=0.9\linewidth]{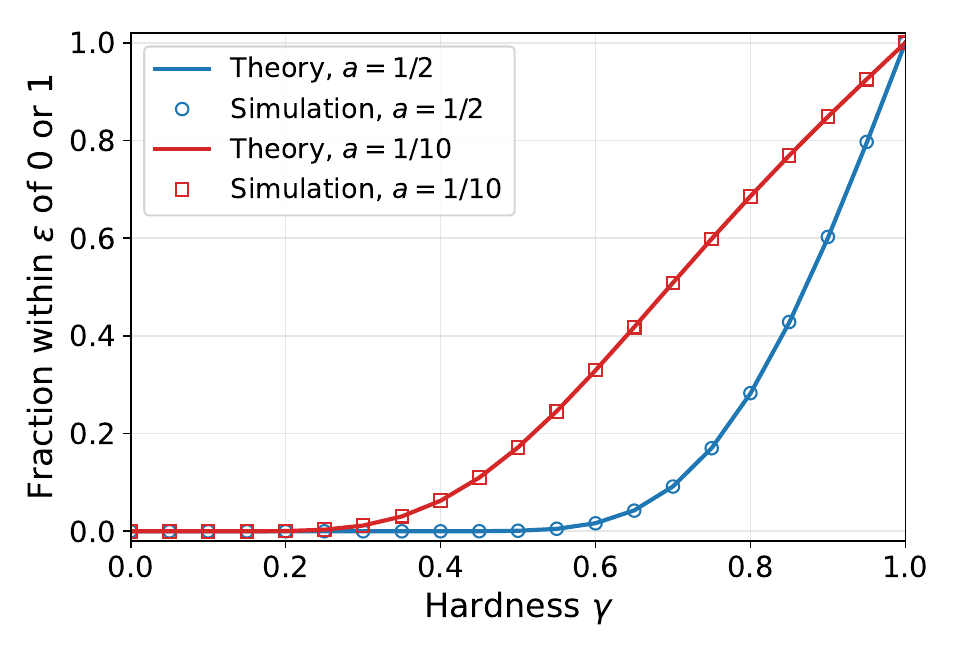}
  \caption{Population prediction and finite-sample saturated fraction on $n=2\times10^5$ Gaussian scores with $\varepsilon=10^{-3}$. Lines show the population model and open markers show simulated masks.}
  \label{fig:supp-mean-field}
\end{figure}

\paragraph{Sensitivity to $\varepsilon$.}
Figure~\ref{fig:supp-eps} repeats the population and simulation comparison for $\varepsilon\in\{10^{-4},10^{-3},10^{-2}\}$. The threshold changes where a finite mask entry is counted as saturated, but the monotone transition and agreement between theory and simulation persist.

\begin{figure}[h]
  \centering
  \includegraphics[width=0.9\linewidth]{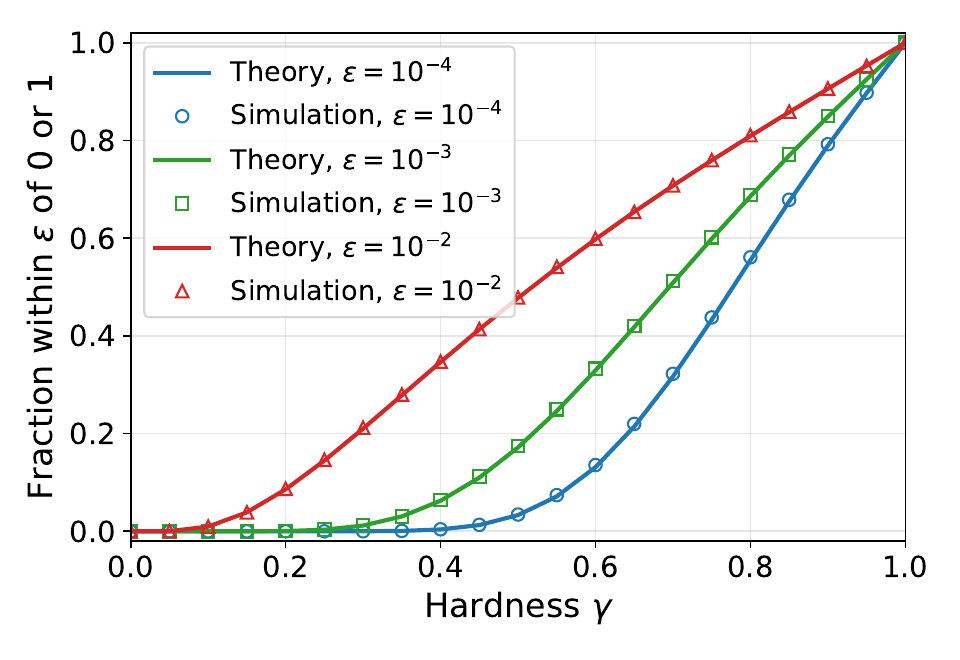}
  \caption{Population prediction and simulation for three numerical saturation thresholds. Each curve uses LaPrune masks with the budget equation enforced to numerical tolerance at $a=1/10$.}
  \label{fig:supp-eps}
\end{figure}

\paragraph{Budget drift.}
Figure~\ref{fig:supp-budget-drift} measures the realized mass of DFTopK against budget-preserving operators on $n=10^4$ Gaussian scores. For $k=n/100$, the DFTopK mass grows from $1.03$ times the target at $\tau=0.05$ to $38.75$ times the target at $\tau=5$. LaPrune and LapSum remain at the requested mass to numerical tolerance.

\begin{figure}[h]
  \centering
  \includegraphics[width=0.9\linewidth]{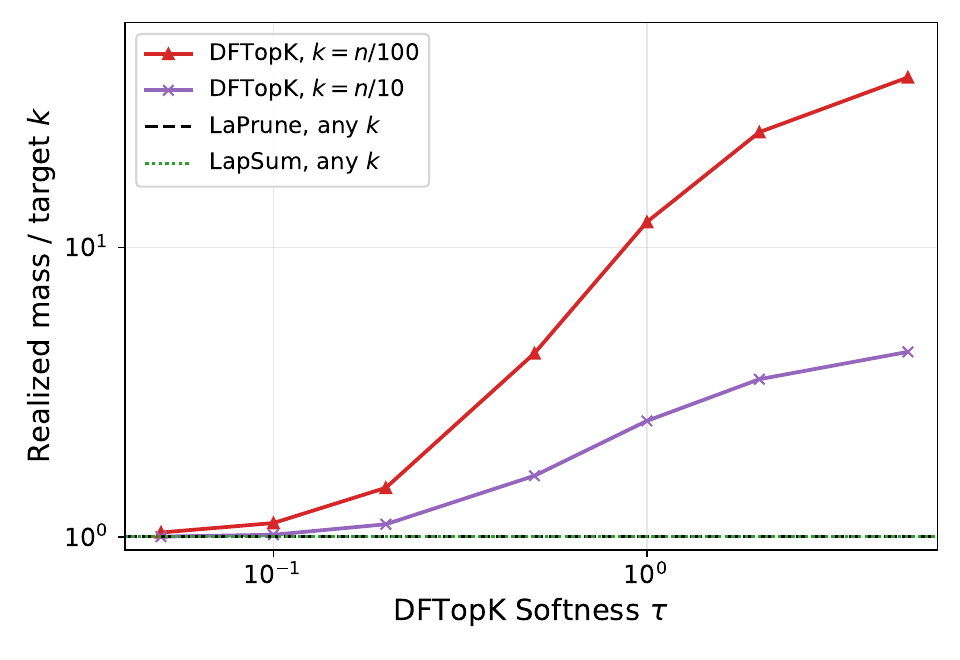}
  \caption{Realized soft-mask mass against DFTopK softness $\tau$ on $n=10^4$ Gaussian scores. DFTopK overshoots sparse budgets as $\tau$ increases. LaPrune and LapSum enforce the budget equation.}
  \label{fig:supp-budget-drift}
\end{figure}

\subsection{Scaling Protocol and Implementations}
\label{app:scaling-protocol}

The scaling benchmark uses one float32 score vector per call (batch size one), with $k=n/16$ and $\gamma=0.9$, on an NVIDIA A100-SXM4-40GB. We use Python 3.9.25, PyTorch 2.8.0 with CUDA 12.8, \texttt{difftopk} 0.2.0, NumPy 2.0.2, and Numba 0.60.0. Every method--dimension pair runs in a fresh subprocess with a 240-second timeout. After one untimed warm-up call, we reset PyTorch's peak-memory counter. CUDA events measure 15 repetitions for $n\le10^5$, seven for $10^5<n\le10^6$, and three for larger $n$. We synchronize the device before reading the events and report the median. Peak memory is \texttt{torch.cuda.max\_memory\_allocated} over the timed calls.

LaPrune uses the compiled pure-PyTorch Newton forward and the custom analytical implicit VJP in this paper. LapSum uses its sort--\texttt{logcumsumexp} barrier solve and custom one-constraint implicit VJP. NeuralSort, SoftSort, and DiffSort use \texttt{difftopk} and PyTorch autograd. The Sander Dykstra and PAV curves use local PyTorch ports of the authors' isotonic formulation with a custom block-averaging backward. The optimal-transport curve uses a local two-bin Sinkhorn implementation with 200 iterations. The first 199 are detached and the final iteration uses autograd. Thus, each curve includes the backward implementation intended for that operator.

The lower measured peak memory of LaPrune (305 MB) than LapSum (544 MB) is an implementation effect, not a consequence of adding fewer constraints. LaPrune is sort-free. Each Newton pass performs elementwise reductions, and its custom backward stores only the score vector and solved scalars. The measured LapSum implementation sorts the full vector and materializes sorted values, indices, and prefix and suffix \texttt{logcumsumexp} work arrays while locating the barrier. These temporary arrays determine its higher peak, whereas LaPrune's 30 Newton passes reuse fixed-size buffers without accumulating per-iteration state.



\subsection{Feature-Selection Budget Sweep}

Figure~\ref{fig:supp-feature-selection} varies the deployment budget over $k\in\{3,5,10,20,40\}$. Each method is evaluated with the same exact hard top-$k$ rule. Annealed LaPrune remains competitive with the fixed-$\gamma$ variant and at or above the measured non-LaPrune budget-preserving baselines throughout the sweep.

\begin{figure}[h]
  \centering
  \includegraphics[width=\linewidth]{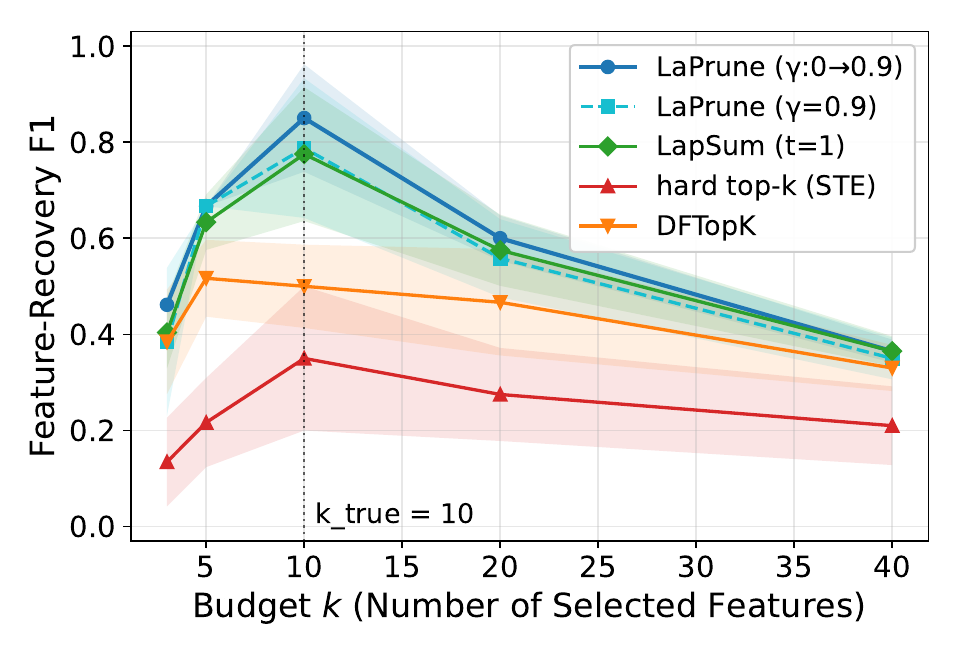}
  \caption{Feature-recovery F1 against deployment budget $k$ for 200-dimensional synthetic data with ten informative features. Curves are means over eight seeds and bands show one standard deviation. The vertical line marks the true feature count.}
  \label{fig:supp-feature-selection}
\end{figure}

\subsection{Sparse-Autoencoder Sweeps}

\paragraph{Architecture, training, and evaluation.}
Let $x\in\mathbb{R}^{512}$ be a standardized frozen ResNet-18 activation. The SAE computes
\[
\begin{aligned}
u&=W_{\mathrm{enc}}(x-b_{\mathrm{pre}})+b_{\mathrm{enc}},\\
h&=p(u)\odot u,\\
\hat x&=W_{\mathrm{dec}}h+b_{\mathrm{pre}}.
\end{aligned}
\]
where $p(u)$ is the method-specific training mask. We minimize the mean squared reconstruction error $\lVert \hat x-x\rVert_2^2/512$ with no sparsity penalty or dead-latent auxiliary loss. We initialize the columns of $W_{\mathrm{dec}}\in\mathbb{R}^{512\times4096}$ from independent standard Gaussians and normalize each to unit norm, set $W_{\mathrm{enc}}=W_{\mathrm{dec}}^\top$ only at initialization, and initialize both biases to zero. The encoder and decoder are then untied. After every Adam update, we renormalize each decoder column to unit norm. Training uses all 50,000 componentwise standardized training activations, Adam with learning rate $10^{-3}$, batches of 4096, and 30 epochs.

For the annealed model at epoch $e$, the exact schedule is
\[
\begin{aligned}
\gamma_e
&=\min\!\left(
0.9,\,
\frac{0.9e}{\lfloor0.8(30-1)\rfloor}
\right).
\end{aligned}
\]
so $\gamma$ reaches $0.9$ at epoch 23 and remains there. Evaluation replaces the training mask with hard top-32 and computes FVU on all 10,000 test activations.
\[
\begin{aligned}
\mathrm{FVU}
&=
\frac{\sum_j\lVert \hat x_j-x_j\rVert_2^2}
{\sum_j\lVert x_j-\bar x_{\mathrm{test}}\rVert_2^2}.
\end{aligned}
\]

The linear probe is trained on hard top-32 codes from all 50,000 training activations and evaluated on the 10,000 test codes. It is a linear $4096\!\to\!100$ softmax classifier with PyTorch's default initialization, trained for 60 epochs with cross-entropy, Adam at learning rate $10^{-2}$, weight decay $10^{-4}$, and batches of 8192. Features are extracted from the avgpool output of an ImageNet-1K-pretrained ResNet-18 and standardized using training-set means and standard deviations.

The main SAE experiment fixes $m=4096$ and $k=32$. Figure~\ref{fig:supp-sae-k} changes the deployment budget at fixed dictionary size. Figure~\ref{fig:supp-sae-m} changes dictionary size at fixed budget.

\begin{figure*}[h]
  \centering
  \includegraphics[width=0.88\textwidth]{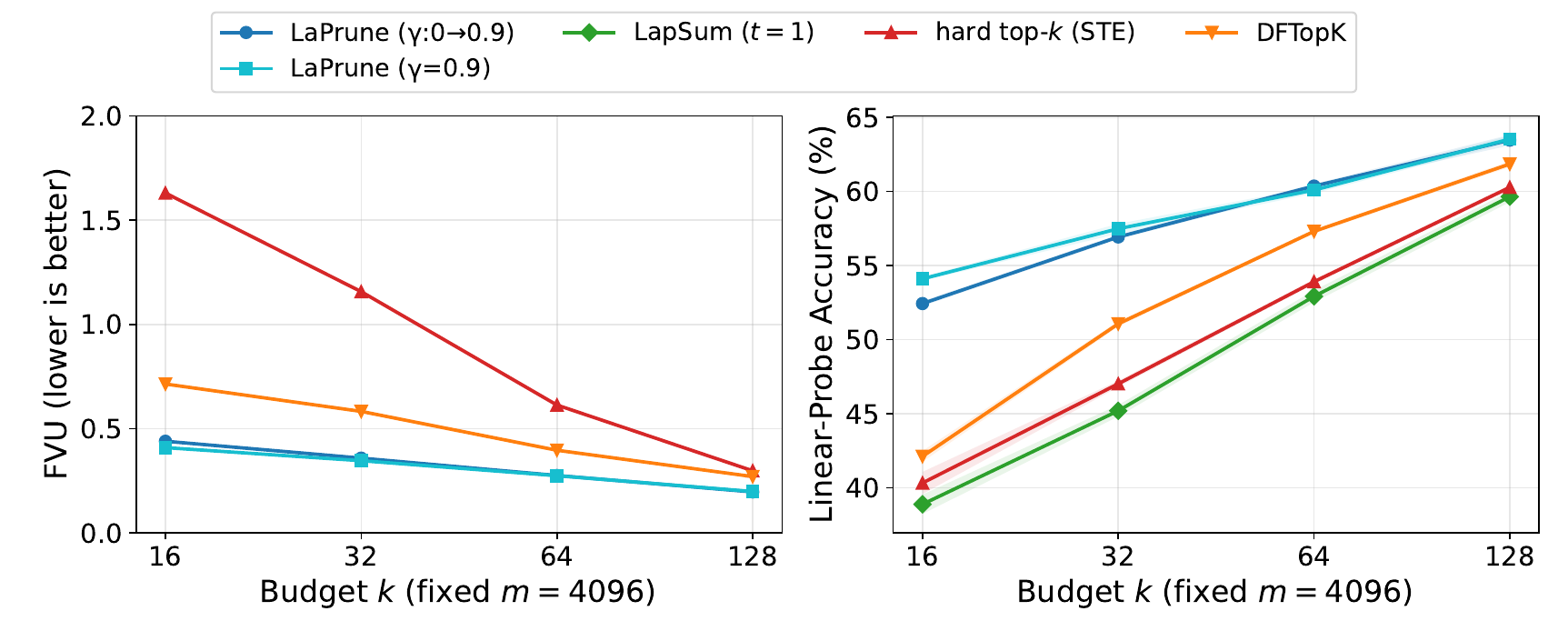}
  \caption{Sparse-autoencoder FVU and linear-probe accuracy against deployment budget $k$ at fixed dictionary size $m=4096$. Curves show means over three seeds and bands show one standard deviation.}
  \label{fig:supp-sae-k}
\end{figure*}

\begin{figure*}[h]
  \centering
  \includegraphics[width=0.88\textwidth]{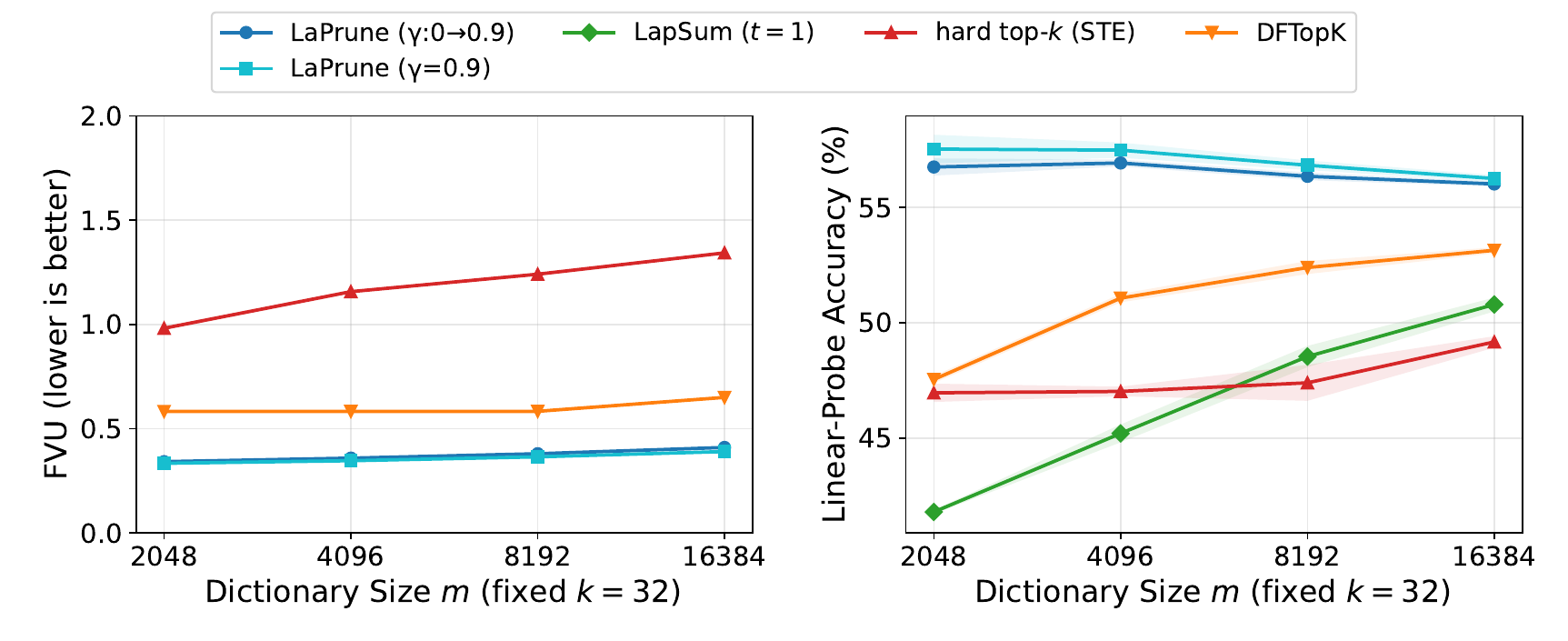}
  \caption{Sparse-autoencoder FVU and linear-probe accuracy against dictionary size $m$ at fixed deployment budget $k=32$. Curves show means over three seeds and bands show one standard deviation.}
  \label{fig:supp-sae-m}
\end{figure*}

\begin{table}[h]
  \centering
  \caption{CIFAR-100 design ablation under the pure top-5 objective. All entries use the same ResNet-18 and 200-epoch protocol.}
  \label{tab:supp-cifar-ablation}
  \small
  \begin{tabular}{lcc}
    \toprule
    Configuration & ACC@1 & ACC@5 \\
    \midrule
    fixed $\gamma=0.3$ & \textbf{72.8} & 82.4 \\
    fixed $\gamma=0.5$ & 71.8 & 86.0 \\
    fixed $\gamma=0.7$ & 67.4 & \textbf{88.0} \\
    fixed $\gamma=0.9$ & 1.0 & 5.0 \\
    fixed $\gamma=0.99$ & 2.0 & 27.1 \\
    \midrule
    anneal $0.1\to0.8$, clip norm 5 & 72.0 & 87.4 \\
    anneal $0.1\to0.8$, no clipping & 71.9 & 86.9 \\
    anneal $0.1\to0.9$, clip norm 5 & 1.0 & 5.0 \\
    \bottomrule
  \end{tabular}
\end{table}

\subsection{CIFAR-100 Top-$k$ Classification}

We also evaluate each operator with a pure top-5 cross-entropy loss using a ResNet-18 trained from scratch on CIFAR-100. For LaPrune, $\gamma$ increases linearly by epoch from $0.1$ at epoch 0 to $0.8$ at epoch 199. Table~\ref{tab:supp-cifar100} reports a single-seed comparison using the competitor softness values specified by the LapSum benchmark. The Smooth top-$k$ row uses the loss of Berrada, Zisserman, and Kumar~\cite{berrada2018smoothlossfunctionsdeep}. Across two seeds, LaPrune achieves ACC@1 $71.93\pm0.07$ and ACC@5 $87.13\pm0.30$.

\begin{table}[h]
  \centering
  \caption{CIFAR-100 classification with a ResNet-18, 200 training epochs, and pure top-5 loss. All methods use the same random seed.}
  \label{tab:supp-cifar100}
  \small
  \begin{tabular}{lcc}
    \toprule
    Method & ACC@1 & ACC@5 \\
    \midrule
    Smooth top-$k$ & 57.3 & 90.3 \\
    NeuralSort & 13.4 & 86.3 \\
    SoftSort & 60.4 & 89.8 \\
    SinkhornSort & 62.4 & 90.2 \\
    DiffSortNets & 62.4 & 90.6 \\
    LapSum & 65.9 & 92.3 \\
    LaPrune ($\gamma{:}\,0.1\!\to\!0.8$) & \textbf{72.0} & 87.4 \\
    \bottomrule
  \end{tabular}
\end{table}

Table~\ref{tab:supp-cifar-ablation} separates the effects of hardness, clipping, and the endpoint of the schedule. Fixed $\gamma$ exposes an ACC@1--ACC@5 trade-off. Training collapses near the binary endpoint. Gradient clipping at norm 5 has little effect for the $0.1\to0.8$ schedule, while extending the schedule to 0.9 collapses under this pure top-5 objective.


\subsection{Reproducibility Details}

All reported random sweeps use fixed integer seeds and store per-seed results before aggregation. The feature-selection models train for 150 epochs with Adam at learning rate $10^{-2}$. The table uses 20 seeds and the budget sweep uses eight. The SAE models train for 30 epochs with Adam at learning rate $10^{-3}$ and batch size 4096. Linear probes train for 60 epochs with Adam at learning rate $10^{-2}$ and weight decay $10^{-4}$. The SAE main table uses eight seeds. Both SAE sweeps use three. CIFAR-100 uses Adam at learning rate $10^{-3.25}$, batch size 100, cosine decay, 200 epochs, and gradient clipping at norm 5 for the main LaPrune schedule. Scaling results use combined forward and backward passes on an NVIDIA A100-SXM4-40GB. The solver stopping, fallback, and warm-start rules are given below.

\subsection{Numerical and Analytical Details}
\label{app:numerics}

\paragraph{Solver Tolerances and Clipping.}
The Newton solve for the joint budget and second-moment system defined in the main paper performs at most 30 iterations and freezes a batch row once
\[
\begin{aligned}
|F_1|&<10^{-7}\max(k,1),\\
|F_2|&<10^{-7}\max(\beta k,1).
\end{aligned}
\]
We report absolute residuals $e_1=|F_1|$ and $e_2=|F_2|$, and relative residuals $e_1/k$ and $e_2/(\beta k)$. After 30 iterations, a row enters the robust fallback if either residual exceeds $10^{-4}$ times the corresponding target clipped below at one. Let $s_r$ be the row-wise score standard deviation. Each Newton increment is clipped to
\[
\begin{aligned}
\Delta b&\in[-(8s_r+1),\,8s_r+1],\\
\Delta\tau&\in[-1.5,1.5].
\end{aligned}
\]
and the updated log-temperature is clipped to $\tau\in[-12,8]$. The barrier itself has no global clipping interval.

\paragraph{Bisection Fallback.}
The fallback uses 48 outer and 48 inner bisection iterations. The outer log-temperature bracket is $\tau\in[-12,8]$. At a candidate $t=e^\tau$, the inner barrier bracket is
\[
\begin{aligned}
b_{\mathrm{lo}}&=r_{\min}-40(t+s_r)-1,\\
b_{\mathrm{hi}}&=r_{\max}+40(t+s_r)+1.
\end{aligned}
\]
The inner solve moves the bracket according to the sign of $\sum_i p_i-k$. The outer solve moves it according to the sign of $\sum_i p_i^2-\beta k$. Both use their fixed iteration limits without a separate stopping tolerance. The default is 48 outer and 48 inner iterations. The CIFAR-100 training runs use 20 and 20, the discretization-control experiment uses 40 and 40, and the supplementary $\varepsilon$-sensitivity sweep uses 24 and 24. The remaining reported experiments retain the default. In a documented diagnostic with 100 independent Gaussian rows, $n=100$, and $a=0.05$, the fallback was required for $0\%$, $27\%$, and $97\%$ of rows at $\gamma=0.3$, $0.9$, and $0.99$, respectively. This increase reflects the ill-conditioning as $t\downarrow0$. Interior values are restricted to the numerical bracket above, so values very near an endpoint may invoke the fallback.

\paragraph{Asymptotic Warm Start.}
For the Gaussian approximation described in the main-paper discussion ``Computing LaPrune: Joint Solves and Implicit Gradients,'' the implementation estimates the density at $b_0$ using the row standard deviation $\hat s$ and $z_a=\Phi^{-1}(1-a)$:
\[
\hat\rho(b_0)=\frac{\phi(z_a)}{\max(\hat s,10^{-9})}.
\]
We clip $\hat\rho(b_0)$ below at $10^{-4}$ and initialize
\[
\begin{aligned}
t_0
&=
\max\left\{
10^{-3},
\frac{4a(1-a)(1-\gamma)}
{3\hat\rho(b_0)}
\right\}.
\end{aligned}
\]
This is a Gaussian plug-in density estimate and therefore has no kernel bandwidth. The optional distribution-free large-$n$ initializer instead draws $m=\lceil\sqrt n\rceil$ scores, uses their empirical $(1-a)$-quantile $\hat b$, and applies a rectangular-window estimator
\[
\begin{aligned}
\hat\rho(\hat b)
&=
\frac{1}{2hm}
\sum_{j=1}^{m}
\mathbf{1}\{|r_j-\hat b|<h\},\\
h
&=
\frac{1}{2m}
\sum_{j=1}^{m}
|r_j-\operatorname{median}(r)|
+10^{-6}.
\end{aligned}
\]
again clipped below at $10^{-4}$.

\paragraph{Analytical Implicit VJP.}
The implementation evaluates the vector--Jacobian product from the implicit-differentiation system in the main-paper discussion ``Computing LaPrune: Joint Solves and Implicit Gradients'' analytically, using the reductions defined there.

For upstream derivatives $g_i=\partial\mathcal{L}/\partial p_i$, let
\[
\begin{aligned}
G_0&=\sum_i g_i d_i,\\
G_1&=\sum_i g_i d_i z_i,\\
D&=S_0U_1-S_1U_0.
\end{aligned}
\]
Then
\[
\begin{aligned}
\frac{\partial\mathcal{L}}{\partial r_i}
&=
\frac{d_i}{t}(g_i-C_i),\\
C_i
&=
\frac{G_0(U_1-S_1p_i)}{D}
+
\frac{G_1(S_0p_i-U_0)}{D}.
\end{aligned}
\]
For completeness, with $c_a=2a+(1-2a)\gamma$,
\[
\begin{aligned}
\frac{\partial\mathcal{L}}{\partial a}
&=-\frac{n}{2D}Q_a,\\
Q_a
&=(S_1c_a-2U_1)G_0\\
&\quad +(2U_0-S_0c_a)G_1,\\
\frac{\partial\mathcal{L}}{\partial\gamma}
&=
-\frac{na(1-a)}{2D}(S_1G_0-S_0G_1).
\end{aligned}
\]
These derivatives apply only to the interior finite-temperature solution.

\paragraph{Closed-Form Moments for Laplace Scores.}
For $\rho=\mathrm{Laplace}(0,1)$ and $b\geq0$, direct integration gives
\begin{align}
\bar M_1
&=\frac{e^{-b}-t^2e^{-b/t}}{2(1-t^2)}=a,
\label{eq:lapclosed1}\\
\bar M_2
&=\frac{(4+t)e^{-b}}{2(1+t)(4-t^2)}
-\frac{t^2e^{-2b/t}}{4(4-t^2)}=\beta a.
\label{eq:lapclosed2}
\end{align}
The apparent singularities at $t=1$ and $t=2$ are removable and are evaluated by continuity. For $b<0$, symmetry gives
\[
\begin{aligned}
\bar M_1(b)&=1-\bar M_1(-b),\\
\bar M_2(b)&=1-2\bar M_1(-b)+\bar M_2(-b).
\end{aligned}
\]

\subsection{Monotonicity and Uniqueness of the Temperature}
\label{app:monotonicity}

\noindent\textit{Proof of the Unique-Temperature Theorem.}
Let
\[
\begin{aligned}
z_i(t)&=\frac{r_i-b(t)}{t},
&w_i(t)&=\sigma'(z_i(t)),\\
\bar z_w(t)&=\frac{\sum_i w_i(t)z_i(t)}
{\sum_i w_i(t)}.
\end{aligned}
\]
Differentiating the budget constraint $\sum_i\sigma(z_i(t))=k$ gives $\sum_iw_i(t)z_i'(t)=0$. Since
\[
z_i'(t)=-\frac{b'(t)+z_i(t)}{t},
\]
we obtain
\[
b'(t)=-\bar z_w(t),
\qquad
z_i'(t)=-\frac{z_i(t)-\bar z_w(t)}{t}.
\]
Consequently,
\begin{align*}
\frac{\mathrm d M_2(t)}{\mathrm d t}
&=2\sum_i\sigma(z_i(t))w_i(t)z_i'(t)\\
&=-\frac{2}{t}\sum_iw_i(t)\sigma(z_i(t))
       \bigl(z_i(t)-\bar z_w(t)\bigr)\\
&=-\frac{2}{t}\left(\sum_iw_i(t)\right)
  \operatorname{Cov}_w\bigl(z(t),\sigma(z(t))\bigr).
\end{align*}
Because $\sigma$ is strictly increasing, the weighted covariance is positive whenever the scores are nonconstant, hence $M_2'(t)<0$. Continuity follows from continuity of the barrier solution. For pairwise distinct scores, the LapSum endpoint limits give $M_2(t)\to k$ as $t\downarrow0$ and $M_2(t)\to ak$ as $t\to\infty$. The intermediate value theorem and strict monotonicity therefore give a unique $t_\gamma$ for every $\gamma\in(0,1)$. \hfill$\square$

\subsection{Proof of the Worst-Case Sparsity Floor}
\label{app:worst-proof}

\noindent\textit{Proof of the Worst-Case Sparsity Floor Theorem.}
Let $U$ be uniform on $[0,1]$ and set $Y=f(U)$. Then
\[
0\leq Y\leq1,
\qquad
\mathbb{E}Y=a,
\qquad
\mathbb{E}Y^2=a\beta.
\]
Write
\[
D:=\mathbb{E}[Y(1-Y)]
=a(1-\beta)
=a(1-a)(1-\gamma).
\]
For any $y\in[0,1]$ and $\varepsilon>0$,
\[
\mathbf{1}\{y<\varepsilon\}
\geq
(1-y)\left(1-\frac{y}{\varepsilon}\right):
\]
the right-hand side is at most one when $y<\varepsilon$ and is nonpositive when $y\geq\varepsilon$. Taking expectations gives
\begin{equation}
\label{eq:worst-direct}
\Pr(Y<\varepsilon)
\geq
1-a-\frac{D}{\varepsilon}
=(1-a)\left(1-\frac{a(1-\gamma)}{\varepsilon}\right).
\end{equation}
Combining Equation~\eqref{eq:worst-direct} with the trivial lower bound zero proves the inequality.

For tightness, set $c=a(1-\gamma)$. If $0<\varepsilon\leq c$, take a two-level nondecreasing step function with values $c$ and $1$, assigning mass
\[
\lambda=\frac{a-c}{1-c}
\]
to the value $1$. Its moments are $a$ and $a\beta$, and it has no value strictly below $\varepsilon$, so the zero branch is tight.

For $c\leq\varepsilon\leq\beta$, assign masses
\[
\begin{aligned}
m_\varepsilon&=\frac{a(1-\beta)}{\varepsilon(1-\varepsilon)},\\
m_1&=a-\varepsilon m_\varepsilon,\\
m_0&=1-m_\varepsilon-m_1
\end{aligned}
\]
to the ordered values $0,\varepsilon,1$. These masses are nonnegative on this interval and satisfy
\[
\varepsilon m_\varepsilon+m_1=a,
\qquad
\varepsilon^2m_\varepsilon+m_1=a\beta.
\]
Moreover,
\[
m_0=(1-a)\left(1-\frac{a(1-\gamma)}{\varepsilon}\right),
\]
so equality holds in Equation~\eqref{eq:worst-direct}. Hence both branches are tight. \hfill$\square$

\paragraph{Extension to $\beta<\varepsilon\leq1$.}
For completeness, let $V=\operatorname{Var}(Y)=a(1-a)\gamma$. Cantelli's inequality gives the tight lower bound
\[
\Pr(Y<\varepsilon)
\geq
\frac{(\varepsilon-a)^2}{(\varepsilon-a)^2+V}.
\]
Equality is attained by a nondecreasing two-level step function with values
\[
u=a-\frac{V}{\varepsilon-a}
\qquad\text{and}\qquad
\varepsilon.
\]
The condition $\varepsilon\geq\beta$ is exactly what ensures $u\geq0$.

\end{document}